\documentclass{article}
\usepackage{kr}

\usepackage{times}
\usepackage{soul}
\usepackage{url}
\usepackage[hidelinks]{hyperref}
\usepackage[utf8]{inputenc}
\usepackage[small]{caption}
\usepackage{graphicx}
\usepackage{booktabs}
\usepackage{natbib}

\usepackage{amsmath,amsthm,amssymb}
\usepackage[ruled,vlined]{algorithm2e}
 \usepackage{tikz}

\usepackage{xcolor}
\newtheorem{definition}{Definition}
\newtheorem{theorem}{Theorem}
\newtheorem{lemma}{Lemma}
\newtheorem{proposition}{Proposition}
\newtheorem{observation}{Observation}
\newtheorem{example}{Example}
\newtheorem{corollary}{Corollary}

\newcommand{\atoms}{\textrm{ATOMS}}
\newcommand{\lang}{\mathcal L}

\newcommand{\data}{\mathcal D}

\newcommand{\U}{\mathbb{N}}
\newcommand{\x}{\vec{x}}
\newcommand{\M}{\mathcal M}
\usepackage{pifont}
\title{Query-Driven Learning for Lifted Reasoning}

\author{%
Luise Ge$^1$
\and
Brendan Juba$^1$
\affiliations
$^1$Washington University in St. Louis
\emails
\{g.luise, bjuba\}@wustl.edu
}

\begin{document}

\maketitle

\begin{abstract}

Learning and reasoning are both essential to intelligent agents, but have traditionally been treated as two separate stages: first learning a knowledge base from data (or taking one as given), and then verifying queries as if the knowledge base were complete. This separation has been shown to be inefficient when the downstream goal is to verify only a small set of queries. Moreover, it neglects the possibility that data and a knowledge base contain partial yet complementary information, requiring an agent to combine both sources to justify or refute a query. Motivated by the need to integrate learning and reasoning more tightly, we study the problem of \textit{query-driven learning for reasoning}. More specifically, we extend prior work to the first-order setting, which is critical for practical applications as the number of ground atoms typically grows rapidly. We begin by asking when validating or refuting a universal query given partial data and knowledge base is even possible, and identify a necessary and sufficient condition called \textit{bounded concealment}. We then propose a quasi-polynomial-time algorithm that solves the task at the lifted level. Finally, we establish soundness and completeness guarantees relative to space-bounded lifted proofs, along with quasipolynomial sample and runtime bounds, demonstrating that lifted query-driven reasoning can be \textit{efficient} even when unrestricted knowledge-base learning is not.
\end{abstract}

\section{Introduction}


Classical proof systems have achieved remarkable success in settings where complete knowledge can be axiomatized in advance, from mathematical theorem proving to hardware and software verification. Their
transparency and rigor, demonstrated in independently checkable proofs, are especially valuable in high-stakes applications. However, their applicability in real-world domains is limited by a strong prerequisite: that all relevant axioms are known \textit{a priori}.

What could an agent do when this prerequisite fails? A natural approach is to first construct an expanded knowledge base from as much data as possible and then reason over it. Yet this “learn everything first” strategy can be computationally prohibitive and often unnecessary: for any particular downstream question, only a small fraction of the learned facts and rules are actually relevant. This motivates a more direct, query-driven alternative: given a specific query, can the agent identify and validate only the additional knowledge needed to justify or refute it?

Concretely, we build on the ``\textit{learning to reason}'' framework of \citet{khardon1997learning}, in particular the extension to partial information with query-driven  explicit learning by \citet{juba2019query}. Their approach takes a background knowledge base and examples from a distribution as input, and attempts to answer a given query by returning an explicit justification when one exists (or rejecting otherwise).

Their results are established at the propositional level. In relational domains with many objects, however, reasoning at the ground level quickly becomes infeasible: both memory and computation scale with the number of ground instances.  For instance, the rule $\forall x. (\mathrm{Bird}(x) \implies \mathrm{Flies}(x))$ over a domain of size $n$ yields $n$ ground implications, and rules over $k$ variables can induce $\Theta(n^k)$ ground ones. This motivates \textit{lifted reasoning}, where one reasons directly with first-order structure. In this work, we ask whether a query-driven learning to reason approach can be carried out at the lifted level: \textit{given a lifted knowledge base and possibly partial data, can we efficiently verify a universally quantified query by producing a lifted proof or reject otherwise?}

Notably, the \textit{implicit learning to reason} framework has made progress toward this goal. In particular, the approach of~\citet{belle2019implicitly} handles universally quantified queries in a limited but effective way: assuming that the relevant counterexamples or supporting evidence are directly ``witnessed'' in the sampled partial models, it estimates the probability that a query clause is valid, but does not produce a rigorous proof. Using the Belle–Levesque grounding trick~\citep{levesque1998,belle2017open}, once a small, critical set of ground instances is observed, the system treats the corresponding universally quantified clause as learned and adds it to the knowledge base. Subsequent universally quantified queries can then be reduced to a propositional unsatisfiability check over a finite grounded fragment.


Our objective is to go beyond this statistical estimation. We aim to prove universally quantified clausal queries, i.e., to certify that the clause holds over the entire domain. This is fundamentally more demanding: unlike falsity, truth of a universal statement typically admits no finite set of witnesses. Consequently, proof requires genuinely lifted reasoning—first-order proof search that accounts for all relevant groundings induced by the background theory and the data—rather than relying on a grounding shortcut that is tailored to finding counterexamples.

At the same time, both the data and the knowledge base may be incomplete. Even setting aside the possibility for having contaminated data, limited observability can make the task impossible: if violations of a clause can be systematically hidden, no procedure can reliably certify universal truth.  Therefore, we first characterize when proof is information-theoretically possible under partial observation. We identify a necessary and sufficient condition named bounded concealment, extending the formulation of \cite{michael2010partial} from learning over finite domains, which requires that if a universally quantified clause is false, then with non-negligible probability the partial observations reveal at least one grounding that witnesses this falsity. Then, under bounded concealment, we show how to reason soundly in a credulous manner: we treat clauses as tentatively valid until we observe explicit counterexamples, with the guarantee that persistent falsity will be exposed with sufficient probability.

Building on this principle, our approach reasons completely at the lifted level. Rather than evaluating ground instances one by one, we apply clause-level proof procedures directly to quantified formulas, enabling generalization across objects and avoiding blow-ups from grounding. This allows us to positively establish universally quantified queries—an ability not provided by the original skeptical, refutation-oriented framework of \citet{belle2019implicitly}.

\paragraph{Our contribution}
We propose the first query-driven framework that learns in service of lifted reasoning, providing simultaneous \textit{efficiency} and \textit{correctness} guarantees. 

\begin{itemize}
\item We broaden \cite{michael2010partial}’s {bounded concealment} condition from finite to infinite domains, and show that our generalized version is a necessary and sufficient condition for query-driven learning to lifted reasoning with incomplete information. 
    
    
    \item We introduce a lifted algorithm that produces formal resolution proofs of the universally quantified clausal query~$\varphi$ from the given knowledge base and an additional, learned $(1 - \epsilon)$-valid knowledge base (i.e., it fails at most an~$\epsilon$ fraction of models under the data distribution), as long as
    \begin{enumerate}
        \item there is a proof of $\varphi$ from the given knowledge and an additional $1$-valid knowledge base, and
        \item the observation process has bounded concealment with respect to the potentially learnable clauses
    \end{enumerate}
    thus certifying that $\varphi$ is $(1-\epsilon)$-valid.
    
    \item We prove that under bounded concealment, our algorithm constructs \emph{treelike lifted resolution proofs} with sample complexity in 
    \[
     O\left( \frac{N^{s-1} \log N + \log(1/\delta)}{\eta \epsilon} \right)
    \]
    and running time in
    \[
    O\left( \frac{N^{2(s-1)} |\varphi| (\log N + \log(1/\delta))}{\eta \epsilon} \right)
    \]
   where $N$ is the number of distinct atoms\footnote{They can be either lifted or ground.} in the knowledge base, $s$ is the space bound (i.e., the maximum number of clauses to be kept at a time when deriving the corresponding treelike resolution proof), $\delta$ bounds the failure probability, and $|\varphi|$ is the length of the binary encoding of the query formula.
\end{itemize}

This is the first approach with \emph{simultaneous} soundness, completeness, and sample and computational efficiency (quasipolynomial in the number of ground relations for lifted query-driven reasoning), enabling scalable and formally checkable inference in complex first-order domains.

\section{Logical Setting}


We begin by formalizing the reasoning component of our task. This requires specifying the logical language used to represent the background knowledge base, the observed data, and the query. Since our goal is lifted reasoning over relational structure, we work in a function-free first-order logical framework (hence purely relational). Specifically, let \(\mathcal{L}\) be a first-order language with:
\begin{itemize}
    \item a countably infinite set of variables \(\{x_1, x_2, x_3, \dots\}\),
    \item predicates such as \(\{P(x_1), \dots, Q(x_1, \dots, x_k), \dots\}\), and
    \item a countably infinite domain \(D := \mathbb{N}\) of \textit{rigid designators} or \textit{names}, which serves as the universe of discourse for quantification.
\end{itemize}
Here all predicates are Boolean-valued. 
An \emph{atom} is an expression of the form \(P(t_1, \dots, t_k)\), where \(P\) is a predicate symbol of arity \(k\), and each \(t_i\) is either a variable or a \emph{name}. We refer to such atoms as \emph{lifted atoms} when they contain variables. In contrast, a \emph{ground atom} is an atom in which all terms \(t_i\) are names. The set of all ground atoms is denoted by \(\atoms\). Literals take the form \(P(t_1, \dots, t_k)\) or \(\neg P(t_1, \dots, t_k)\), where \(P(t_1, \dots, t_k)\) is an atom.

\paragraph{Syntax.}
Like in standard FOL, reasoning over the full fragment of $\lang$ is undecidable.  
In particular, we will be interested in learning and reasoning with universally quantified clauses, which form both the \textit{knowledge base}\footnote{Hence one can think of a knowledge base as a CNF.} and the \textit{lines} in our resolution proofs. 

A universally quantified clause (or $\forall$-clause)\footnote{This is different from the definition of $\forall$-clause in~\citet{belle2019implicitly} as they allow equality formulas to restrict the domain.} in our setting is a formula of the form:
$\forall \vec{x}.L_1(\vec{x}) \vee \cdots \vee L_n(\vec{x})$
where each $L_i(\vec{x})$ is a literal over a subset of the universally quantified variables $\vec{x}$. Throughout the paper, we will omit the quantifier and simply write such clauses as $\varphi(\vec x)$ or $\varphi(x)$ (if there is a single variable) when the quantification over free variables is unambiguous from context.

We will refer to elements of \(\atoms\) as \emph{propositions}, and to ground formulas as \emph{propositional formulas}. We use 
\(\alpha, \beta, \varphi, \psi\) to denote ground formulas.

\paragraph{Semantics.}

A \(\lang\)-model \(M\) is a \(\{0,1\}\)-assignment to the elements of \(\atoms\). Using \(\models\) to denote satisfaction, the semantics of a formula \(\varphi \in \lang\) is defined inductively, with the following convention:
\begin{itemize}
    \item Quantification is substitutional over all names in \(\U\): Let $\x$ denote the (ordered) tuple of free variables in $\varphi$.  
    \(M \models \forall \vec{x}. \varphi(\vec{x})\) iff \(M \models \varphi(\vec{a})\) for all \(\vec{a} \in \U^k\), where \(k= |\x|\) .
\end{itemize}

\begin{definition}[Validity of $\forall$-clause]
    Let the set of all models be denoted by \(\M\). We say a $\forall$-clause is \emph{valid} if it is satisfied by every model in \(\M\).

\end{definition}

\subsection{Resolution}
We use (first-order) resolution as our underlying proof system. \text{Resolution} is a classical proof system that operates over disjunctions of literals via the cut rule:
\[
\frac{A \lor L \quad B \lor \neg L}{A \lor B},
\]
optionally augmented with the \emph{weakening} rule, which allows arbitrary literals to be added to clauses:
\[
\frac{A}{A \vee B}.
\]

These rules form the deductive backbone of our system. In our setting, resolution operates over universally quantified clauses—formulas of the form $\forall x\, (\psi(x) \rightarrow \chi(x))$, or equivalently, $\forall x\, (\neg \psi(x) \lor \chi(x))$. This \emph{lifted} formulation allows reasoning to proceed symbolically, without grounding clauses to all possible instantiations. It supports generalization across individuals in an infinite domain and is central to the efficiency and interpretability of our method.


In a general resolution derivation, each derived clause may be used as a premise in multiple later inferences. Consequently, the dependency structure induced by “which earlier lines are used to derive which later lines” forms a directed acyclic graph (DAG). A \emph{treelike} resolution proof is the special case in which no derived clause is reused: each derived line is used as a premise at most once. Equivalently, the underlying dependency DAG is a tree (rooted at the final contradiction).

Following the structure of the classical DPLL procedure~\cite{davis1960computing,davis1962machine}, which constructs treelike refutations via variable splits, our method performs clause-level case analysis to build lifted proofs. This analogy extends further: DPLL with bounded memory corresponds to space-bounded treelike proofs, where only $s$ clauses can be held at once, and forgotten clauses must be re-derived~\cite{Kullmann1999,esteban2001space}. Our lifted algorithm follows this model, combining case splits with clause reuse under a space constraint.

Importantly, all DPLL-like algorithms (including ours) produce normal resolution proofs:

\begin{definition}[Normal]
We will say that a resolution proof is {\em normal} if in its corresponding DAG:
\begin{itemize}
\item All outgoing edges from Cut nodes are directed to Cut nodes.
\item The clauses labeling any path to the sink from a Cut node contain literals
using every variable along the path.
\item A given variable is used in at most one cut step and at most one weakening
step along every path from a source to a Cut node.
\end{itemize}
\end{definition}

Let $N$ denote the number of atoms. Let $s$ be the space of a treelike resolution proof, which corresponds to the \emph{rank} or Horton–Strahler number of its proof tree \citep{ansotegui2008measuring}. A key structural fact, due to~\citet{ehrenfeucht1989learning}, bounds the size and number of such proofs in terms of $s$ and $N$.

\begin{lemma}[\citeauthor{ehrenfeucht1989learning}, \citeyear{ehrenfeucht1989learning}]
\label{lem:bound}
Let $k$ be the number of nodes in a normal space-$s$ treelike resolution proof over $N$ atoms, with $N \geq s \geq 1$. Let $e$ denote the base of $\ln$. Then:
\begin{enumerate}
    \item The proof size $k$ is bounded by $2^s - 1 \leq k \leq 2\left(\tfrac{eN}{s - 1}\right)^{s - 1}$.
    \item The number of distinct normal space-$s$ treelike proofs that do not use weakening is at most $(8N)^{(eN/(s-1))^{s-1}}$.
\end{enumerate}
\end{lemma}

Note that Lemma~\ref{lem:bound} implies the space parameter is logarithmic in the proof size. Consequently, the sample complexity and running time bounds we obtain, with 
$N^s= N^{O(\log N)}$, are quasipolynomial in $N$. These combinatorial bounds also play a central role in our sample complexity analysis. They quantify the expressive capacity of bounded-space proofs and guide how much data is needed to safely validate axioms.



\section{PAC-semantics}
We now turn to the learning component of our task. The key difficulty is that induction from data does not support classical, worst-case guarantees: from finitely many samples, we generally cannot certify that a nontrivial first-order statement holds in \emph{all} interpretations of the language. As a result, conclusions supported by learning must be interpreted probabilistically.

More specifically, we cannot expect learned statements to be valid in the classical (Tarskian) sense, since PAC learning does not guarantee that a formula holds under every possible model. Instead, it guarantees generalization: with high probability over the sampled data, a learned formula $\varphi$ will agree with fresh, unseen examples drawn from the same underlying distribution. This viewpoint is formalized by the PAC semantics of \citet{valiant1999robust}, which provides a principled, distribution-dependent notion of approximate satisfaction that is weaker than classical truth.

In the classical propositional setting, one typically works with a finite language over \(n\) propositional symbols, so models correspond to truth assignments in $\{0,1\}^n$. Our setting, by contrast, involves first-order relational structure and an unbounded collection of ground atoms, so the corresponding space of interpretations is inherently infinitary. Following \citet{belle2019implicitly}, we  adopt PAC semantics in a way that evaluates formulas relative to a data-generating distribution and a partial observation process:

\begin{itemize}
    \item We define a \(\sigma\)-algebra \(\Sigma_{\mathcal{M}}\) over the set of all models \(\mathcal{M}\), representing measurable collections of (possibly infinitary) events.
    \item We introduce a probability measure \(\Pr: \Sigma_{\mathcal{M}} \to [0,1]\), defining a distribution over models.
\end{itemize}

This allows us to define a notion of \((1 - \varepsilon)\)-validity, capturing the idea that a formula holds in most models drawn from the distribution.

\begin{definition}[PAC Validity of a Formula]
Let \([\![ \varphi ]\!]\) denote the set of models satisfying a closed formula \(\varphi \in \mathcal{L}\):
\[
[\![ \varphi ]\!] := \{ M \in \mathcal{M} \mid M \models \varphi \}.
\]
Given a probability distribution \(\mathcal{P}\) over \(\mathcal{M}\), we say that \(\varphi\) is \emph{\((1 - \varepsilon)\)-valid} with respect to \(\mathcal{P}\) if
\[
\Pr_{M \sim \mathcal{P}}([\![ \varphi ]\!]) \geq 1 - \varepsilon.
\]
If \(\varepsilon = 0\), we say that \(\varphi\) is \emph{perfectly valid}.
\end{definition}

\begin{example}
Let \(\varphi(x): \forall x\, \text{Squirrel}(x) \rightarrow \text{Grey}(x)\) assert that all squirrels in a forest are grey. Suppose this generally holds—e.g., due to genetic inheritance—but rare mutations occasionally produce yellow squirrels. So across all possible worlds, most worlds contain no yellow squirrels at all; but a small fraction may include one or more. In this case, \(\varphi\) is not classically valid, but can be \((1 - \varepsilon)\)-valid for \(\varepsilon = 0.1\):
\[
\Pr_{M \sim \mathcal{P}}[M \models \varphi] \geq 0.9.
\]

\end{example}

In classical logic, validity is most often used to establish \textit{entailment}—i.e., whether a knowledge base entails a given query. PAC validity captures that \(\varphi\) holds in most models under \(\mathcal{P}\), enabling probabilistic reasoning under uncertainty. In this PAC setting, one may naturally ask: how does this notion of inference translate when formulas are only approximately valid? As observed by~\citet{Juba2013Implicit} in the propositional setting, the union bound provides a natural analogue of classical reasoning under PAC semantics. That is, if a set of formulas are each individually \((1 - \varepsilon_i)\)-valid with respect to $\mathcal P$, and their conjunction classically entails a ground formula \(\varphi\), then the query is also PAC-valid with controlled error.

The result also holds for lifted queries such as $\forall$-clause  since PAC validity of a \(\forall\)-clause requires satisfaction under all groundings, and classical entailment preserves validity under substitution.

\begin{proposition}
\label{classical-inf-bound}
Let \(\psi_1, \ldots, \psi_k\) be \(\forall\)-clauses such that each \(\psi_i\) is \((1 - \varepsilon_i)\)-valid under a common distribution \(\mathcal{P}\), for some \(\varepsilon_i \in [0,1]\). Suppose that \(\{\psi_1, \ldots, \psi_k\} \models \varphi\), where \(\varphi\) is either a ground formula or a \(\forall\)-clause. Then \(\varphi\) is \((1 - \varepsilon')\)-valid under \(\mathcal{P}\), where \(\varepsilon' = \sum_i \varepsilon_i\).
\end{proposition}

\section{Observability and Learnability}

We first recall the idealized setting of \citet{khardon1997learning}, in which examples are fully observed. That is, each sample specifies a complete truth assignment to all ground atoms in the language, so empirical frequencies can be used to estimate the probability of violations. In this case, falsity of a universal clause admits finite witnesses, and sample complexity follows from a simple concentration argument.

\begin{observation}[Full Observability]\label{obs:full-observability}
Let $\varphi$ be a fixed universally quantified clause evaluated on complete examples drawn i.i.d.\ from a distribution $\mathcal P$. If 
\[
m \;\ge\; \frac{1}{\varepsilon}\ln\!\left(\frac{1}{\delta}\right)
\]
and no violation is observed in $m$ examples, then $\varphi$ is $(1-\varepsilon)$-valid with probability at least $1-\delta$.
\end{observation}

\begin{proof}
    Let $p$ be the probability that a random model drawn from $\mathcal{P}$ violates the clause $\varphi$. We aim to show that if $p > \varepsilon$, the probability of observing no violations in $m$ independent examples is at most $\delta$.

Since the examples $M^{(1)}, \dots, M^{(m)}$ are drawn i.i.d., the probability that $\varphi$ is satisfied by all $m$ examples is given by:
\[
\mathbb{P}(\forall i, M^{(i)} \models \varphi) = (1 - p)^m
\]
Assuming the clause is not $(1-\varepsilon)$-valid (i.e., $p > \varepsilon$), we can upper-bound this probability:
\[
(1 - p)^m < (1 - \varepsilon)^m
\]
Using the standard inequality $1 - x \le e^{-x}$ for $x \in [0, 1]$, we obtain:
\[
(1 - \varepsilon)^m \le e^{-m\varepsilon}
\]
To ensure the probability of failing to observe a violation is bounded by our confidence parameter $\delta$, we require:
\[
e^{-m\varepsilon} \le \delta
\]
Taking the natural logarithm of both sides and solving for $m$, we obtain that
if $m \ge \frac{1}{\varepsilon}\ln(\frac{1}{\delta})$ and no violation is observed, $\varphi$ is $(1-\varepsilon)$-valid with probability at least $1-\delta$.
\end{proof}

\subsection{Observation Model}

The fully observed setting above is a simple starting point, but it is rarely realistic. In our formalization, a complete example specifies a truth value for every ground atom, whereas real data typically reveals only a small, task-dependent subset of predicates and ground instances. Many facts of interest are latent, costly to measure, or not directly observable at all. This incompleteness is precisely what makes both induction from data and deduction from background knowledge necessary. 
 Handling corrupted or adversarially perturbed observations is also important, but we defer that extension to future work.
 
Hence, let us treat partial observability as an explicit part of the model: the learner observes only a masked view of an underlying (unobserved) $\mathcal{L}$-model.

\begin{definition}[Partial Model]
 A \emph{partial model} \(\rho\) is a function \(\rho: \atoms \to \{1, 0, *\}\), where \(1\) and \(0\) indicate truth and falsity, respectively, and \(*\) denotes an unobserved (masked) value.

We say that a partial model \(\rho\) is \emph{consistent} with a full \(\mathcal{L}\)-model \(M \in \mathcal{M}\) if, for all \(A \in \atoms\), whenever \(\rho(A) \neq *\), it holds that \(\rho(A) = M(A)\). Let \(\mathcal{N}\) denote the set of all partial models.
\end{definition}

\begin{definition}[Masking]
A \emph{mask} is a function \(\theta: \mathcal{M} \to \mathcal{N}\) that maps full models to partial models. It must satisfy the property that for all \(M \in \mathcal{M}\), the resulting partial model \(\theta(M)\) is consistent with \(M\), and only finitely many atoms are mapped to \(\{0,1\}\); all others are masked as \(*\).

A \emph{masking process} \(\Theta\) is a mask-valued random variable—i.e., a distribution over masking functions. Applying \(\Theta\) to a distribution \(\mathcal{P}\) over full models induces a distribution \(\Theta(\mathcal{P})\) over partial models.
\end{definition}

\subsection{Necessary Condition}

A central difficulty is that, under unrestricted masking, violations of a universally quantified clause may be systematically concealed. In the extreme, the observation process can make truth and falsity observationally indistinguishable, so that even refutation is impossible from data alone.

\begin{observation}[Zero Observability]\label{obs:zero-observability}
Fix a universally quantified clause $\psi$ (e.g., $\forall x\,\psi(x)$). Suppose the masking process $\Theta$ has the following property: for every full model $M$,
\[
M \models \neg \psi \quad\Longrightarrow\quad \Pr\bigl[\rho \models \neg \psi \ \bigm|\ \rho=\Theta(M)\bigr]=0,
\]
i.e., whenever $\psi$ is false in $M$, the induced partial model never reveals a detectable violation of $\psi$.
Then samples from $\Theta(\mathcal{P})$ cannot be used to upper-bound the violation probability of $\psi$ under $\mathcal{P}$: for any $\epsilon\in(0,1)$ there exist distributions $\mathcal{P}_0,\mathcal{P}_1$ such that
\[
\Pr_{M\sim \mathcal{P}_0}[M\models \neg\psi]=0,\qquad 
\Pr_{M\sim \mathcal{P}_1}[M\models \neg\psi]\ge \epsilon,
\]
yet $\Theta(\mathcal{P}_0)$ and $\Theta(\mathcal{P}_1)$ are identical. In particular, no algorithm observing only $\rho\sim \Theta(\mathcal{P})$ can distinguish whether $\psi$ is always true or frequently false.
\end{observation}

Observation~\ref{obs:zero-observability} motivates a minimal restriction on the masking process: false formulas must be detectable with non-negligible probability. This definition extends the finite version from \citet{michael2010partial}, which essentially only considered the special case of supervised learning, corresponding to formulas $\varphi(\vec{x},b)$ of the form $f(\vec{x})=b$ for some (simple) formula $f$, where $\vec{x}$ and $b$ are propositional. Here, we extend it to more general formulas in a first-order relational language.

\begin{definition}[Bounded Concealment]
Let \(\mathcal{C}\) be a set of (lifted) formulas and \(\mathcal{P}\) a distribution over \(\mathcal{L}\)-models. A masking process \(\Theta\) is said to be at most \((1 - \eta)\)-\emph{concealing} with respect to \(\mathcal{C}\) and \(\mathcal{P}\) if
 $$\forall \varphi \in \mathcal{C}, \Pr_{M \sim \mathcal P, \rho=\Theta(M)} [  \rho \models \neg \varphi \  | \ M \models \neg \varphi ] \ge \eta.$$
That is, when a formula \(\varphi\) is false in a full model \(M\), there is at least \(\eta\) probability that its falsity is detectable from the corresponding partial model \(\rho = \Theta(M)\).
\end{definition}

\begin{example}[Bounded Concealment Over a Clause Set]
Let \(\mathcal{C}\) be a set of universally quantified clauses describing regularities in a forest domain:
\begin{align*}
\varphi_1 &: \forall x\, \text{Squirrel}(x) \rightarrow \text{Furry}(x), \\
\varphi_2 &: \forall x\, \text{Squirrel}(x) \rightarrow \text{Grey}(x), \\
\varphi_3 &: \forall x\, \text{Squirrel}(x) \rightarrow \text{Small}(x).
\end{align*}
These formulas are mostly true in the underlying distribution \(\mathcal{P}\) due to genetic inheritance, but may be occasionally violated due to mutation.

Now suppose an $\eta$ fraction of the population, chosen uniformly at random, has all predicates observed. Then \(\Theta\) is \((1 - \eta)\)-concealing w.r.t. \(\mathcal{C}\) if, whenever a clause \(\varphi_i\) is violated in a full model \(M\), there is at least \(\eta\) probability that the violation is visible in the partial model \(\rho = \Theta(M)\).

Thus, even under limited observability, clause violations can be detected with bounded probability.
\end{example}

The bounded concealment assumption justifies lifted reasoning: we are allowed to assign high confidence to universally quantified formulas when no counterexamples are observed. 
The following ``discounting'' theorem of \citet{michael2010partial} makes this precise; although Michael's theorem was stated and proved for the special case of supervised learning with propositional variables as described above, the proof carries over.


\begin{theorem}\label{thm:discount} \cite[Theorem 2.2]{michael2010partial}  
    For any distribution $\mathcal P$ over models, masking process $\Theta$, and class $\mathcal{C}$ of formulas, if $\Theta$ is at most $(1 - \eta)$-concealing with respect to $\mathcal{C}$ and $\mathcal P$, and $\varphi \in \mathcal{C}$ is not $(1-\varepsilon)$-valid, then  
    \[
    \Pr_{\rho \sim \Theta(\mathcal P)}[\rho \models \neg \varphi ] \geq \eta \varepsilon.
    \]
    Conversely, if $\Theta$ is not at most $(1 - \eta)$-concealing with respect to $\mathcal{C}$ and $\mathcal P$, then there exists $\varphi \in \mathcal{C}$ such that  
    \begin{align*}
        \Pr_{\rho \sim \Theta(\mathcal P)}[\rho \models \neg \varphi ] \leq \eta \Pr_{M \sim \mathcal P}[M\models \neg \varphi ],
    \end{align*}
    where, note, $\varphi$ is $(1 - \Pr_{M \sim \mathcal P}[M\models \neg \varphi])$-valid.
\end{theorem}

Therefore, a bounded concealment condition is necessary for learning universal clauses.
In particular, we note that it almost immediately gives a lower bound on the number of examples required for the more basic task of certifying $(1-\epsilon)$-validity (ignoring running time) for varying values of the concealment parameter $\eta$:
\begin{corollary}
    Suppose that some function $A$, given any formula $\varphi$ in the class of formulas $\mathcal{C}$ and $m(\varphi,\eta,\epsilon,\delta)$ partial examples from any distribution on masked examples $\Theta(\mathcal{P})$ for an at most $(1-\eta)$ concealing $\Theta$ with probability $1-\delta$:
    \begin{itemize}
        \item If $\varphi$ is perfectly valid with respect to $\mathcal{P}$, outputs ``accept''.
        \item If $\varphi$ is not $(1-\epsilon)$-valid with respect to $\mathcal{P}$, outputs ``reject''.
    \end{itemize}
    Suppose that for some distribution $\mathcal{P}^*$ and $(1-\eta)$-concealing masking process $\Theta^*$, $\Theta^*$ is not at most $(1-\eta-\gamma)$-concealing with respect to $\mathcal{C}$ and $\mathcal{P}^*$ for any $\gamma>0$. Then for all $\varphi$, $\delta$, and $\epsilon$ such that $\eta\epsilon \leq 1/2$,
    \[
    m(\varphi,\eta,\epsilon,\delta)\geq \frac{\ln(1/2\delta)}{2\eta\epsilon}.
    \]
\end{corollary}
\begin{proof}
    We first observe that for any $\varphi\in\mathcal{C}$, there is a distribution $\tilde{\mathcal{P}}$ and masking process $\tilde{\Theta}$ such that
    \begin{itemize}
        \item $\varphi$ is perfectly valid on $\tilde{\mathcal{P}}$ and
        \item $\tilde{\Theta}(\tilde{\mathcal{P}})$ is the same distribution on partial examples $\rho$ as $\Theta^*(\mathcal{P}^*)$ conditioned on the partial example $\rho$ satisfying $\rho\not\models\neg\varphi$.
    \end{itemize}
    Now, by the first correctness condition, for any $\epsilon$ and $\delta$, when $A$ is given $m(\varphi,\eta,\epsilon,\delta)$ examples from $\tilde{\Theta}(\tilde{\mathcal{P}})$, since $\varphi$ is perfectly valid on $\tilde{\mathcal{P}}$, it must accept with probability $1-\delta$.

    Now, let $\gamma>0$ be given and $\varphi$ be a $1-\epsilon$-valid formula for $\Theta^*$ and $\mathcal{P}^*$ satisfying the second condition of Theorem \ref{thm:discount}, where $\Theta^*$ is $(1-\eta)$-concealing but not at most $(1-\eta-\gamma)$-concealing. Observe that in each example, the probability that $\rho\not\models\neg\varphi$ is at least $1-(\eta+\gamma)\epsilon$, and hence, the probability that $m(\varphi,\eta,\epsilon-\gamma,\delta)$ independent examples all satisfy $\rho\not\models\neg\varphi$ is at least
    $(1-(\eta+\gamma)\epsilon)^{m(\varphi,\eta,\epsilon-\gamma,\delta)}$. In this case, these are $m(\varphi,\eta,\epsilon-\gamma,\delta)$ examples drawn from $\Theta^*(\mathcal{P}^*)$ conditioned on $\rho\not\models\neg\varphi$, which is identical to $\tilde{\Theta}(\tilde{\mathcal{P}})$ by construction. So by a union bound, the probability that $A$ does not output accept on the examples from $\Theta^*(\mathcal{P}^*)$ is at most
    \[
    \delta+1-(1-(\eta+\gamma)\epsilon)^{m(\varphi,\eta,\epsilon-\gamma,\delta)}.
    \]
    But, by our choice of $\varphi$ which is not $1-\epsilon+\gamma$-valid with respect to $\mathcal{P}^*$, where $\Theta^*$ is $(1-\eta)$-concealing, the second correctness condition for $A$ requires that this probability is at least $1-\delta$. Since $1-x+x^2\geq e^{-x}$ for $x\geq 0$, and $x-x^2\geq x/2$ for $x\leq 1/2$ we have
    \begin{align*}
        1-\delta &\leq \delta+1-(1-(\eta+\gamma)\epsilon)^{m(\varphi,\eta,\epsilon-\gamma,\delta)})\\
        \Rightarrow 2\delta &\geq e^{-2(\eta+\gamma)\epsilon\cdot m(\varphi,\eta,\epsilon-\gamma,\delta)}\\
        m(\varphi,\eta,\epsilon-\gamma,\delta) &\geq\frac{\ln{1/2\delta}}{2(\eta+\gamma)\epsilon}.
    \end{align*}
    Note that without loss of generality, $m(\varphi,\eta,\epsilon,\delta)$ is monotonically decreasing with respect to $\epsilon$.
    Since $\gamma>0$ was arbitrary, letting $\gamma\to 0$ yields the claimed bound.
\end{proof}

\noindent
Indeed, as $\eta\to 0$, $m\to\infty$, recovering Observation \ref{obs:zero-observability}. More generally, if a method purports to certify $1-\epsilon$ validity with confidence $1-\delta$ from $m$ examples, then solving for $\eta$ in the bound, the method implicitly assumes that the masking is at most $(1-\eta)$-concealing for the corresponding value. We remark that our method will asymptotically match this inherent $\frac{\log 1/\delta}{\eta\epsilon}$ dependence on $\eta$, $\epsilon$, and $\delta$ in the number of examples required (and hence also running time). We remark that we are solving a slightly richer task than assumed in our statement here, which requires us to output a proof of a query from a given knowledge base and some $(1-\epsilon)$-valid formulas, which will indeed require some additional dependence on the complexity of those formulas in the number of examples used. But this is indeed a strictly richer formulation: in the special case where the knowledge base is empty, we can construct an algorithm that just outputs ``accept'' when a proof is found, solving the above task formulation. It is therefore bound by the same requirements.
\section{Query-Driven Algorithm}

We now present an algorithm for lifted query-driven reasoning under partial observability. Given a universally quantified clause $\varphi$, a symbolic knowledge base $\mathcal{K}$, and a set of partial models (i.e., masked observations) as data, the algorithm either returns a lifted resolution proof of $\varphi$ or soundly rejects it. The proof, if found, may use both symbolic axioms from $\mathcal{K}$ and new clauses validated statistically from the data.

The algorithm performs a space-bounded proof search, constructing lifted resolution trees with bounded memory. It leverages the bounded concealment assumption to learn new axioms from data and interleaves them with symbolic reasoning to justify universal claims efficiently.

The search proceeds in three phases:
\begin{itemize}
  \item \textbf{Validation}: If no observed partial model falsifies any grounding of $\varphi$, we accept $\varphi$ as PAC-valid and return a proof asserting it.
  \item \textbf{Weakening}: If $\varphi$ is subsumed by a clause in $\mathcal{K}$ via variable renaming, we derive it using weakening.
  \item \textbf{Recursive Resolution}: Otherwise, we branch on pivot literals drawn from $\mathcal{K}$ to recursively construct a lifted proof of $\varphi$, bounded by a space parameter $s$.
\end{itemize}

To make this procedure precise, we now define two core concepts used in the resolution step: \emph{subsumptive weakening} and \emph{pivot literals}.

\begin{definition}[Subsumptive Weakening]
Let $\alpha(y_1,\dots,y_\ell)$ be a lifted clause in $\mathcal{K}$, and let
$\varphi(x_1,\dots,x_k)$ be the lifted clause query. We say \emph{$\alpha$ subsumes $\varphi$ by weakening} if there exists a (not necessarily injective) renaming
\[
  \sigma:\{y_1,\dots,y_\ell\}\;\longrightarrow\;\{x_1,\dots,x_k\}
\]
such that for each literal $L$ of $\alpha$, the renamed literal
\[
  L\bigl(\sigma{(y_1)},\dots,\sigma{(y_\ell)}\big)
\]
appears among the disjuncts of $\varphi$.  
\end{definition}

\begin{example} Let $\alpha(x,y) = A(x) \vee B(y)$ and let $\varphi(x) = A(x) \vee B(x) \vee C(x)$. Then $\alpha$ subsumes $\varphi$ by weakening via $\sigma(x) = \sigma(y) = x$. This illustrates how non-injective substitutions allow us to derive special cases from more general rules in the knowledge base, enabling clause reuse even when variable bindings collapse.
\end{example}

\begin{definition}[Pivot Literals for \(\varphi\)]
Let \(\varphi(x_1,\dots,x_k)\) be the query with \(k\) universally quantified variables.  For any literal \(L(\vec{y})\) containing universally quantified variables occurring in \(\mathcal{K}\), each literal of the form $L({\sigma(\vec{y})})$ for some renaming $
  \sigma:\{y_1,\dots,y_\ell\}\;\longrightarrow\;\{x_1,\dots,x_k\}
$
is called a \emph{pivot literal over} \(\varphi\).  
\end{definition}

\begin{example} Suppose $\mathcal{K}$ contains the clause $\forall z. A(z)$, and the query is $\varphi(x, y) = B(x) \vee C(y)$. Then both $A(x)$ and $A(y)$ are considered pivot literals for $\varphi$, corresponding to renamings $\sigma(z) = x$ and $\sigma(z) = y$, respectively. Even though they originate from the same clause in $\mathcal{K}$, they are treated as distinct literals because $x$ and $y$ are different variables in the scope of $\varphi$.
\end{example}

\begin{algorithm}[t]
\caption{\textsc{LiftedLearn+Search}}
\label{alg:lifted_learn_search}
\KwIn{
  Knowledge base $\mathcal{K}$, space bound $s \ge 1$, query clause $\varphi(x_1,\dots,x_k)$,\\
  Partial models $\rho^{(1)}, \dots, \rho^{(m)}$
}
\KwOut{Lifted resolution proof of $\varphi$ or \texttt{None}}

\If{$\forall i,\; \rho^{(i)} \not\models \neg\varphi$}{
  \Return proof asserting $\varphi$ from statistical validation
}

\ForEach{clause $\alpha(\vec{y}) \in \mathcal{K}$}{
  \If{there exists substitution $\sigma : \{y_j\} \rightarrow \{x_i\}$ s.t.\ $\sigma(\alpha) \subseteq \varphi$}{
    \Return weakening derivation of $\varphi$ from $\sigma(\alpha)$
  }
}

\If{$s > 1$}{
  \ForEach{pivot or ground literal $\ell$ relevant to $\varphi$ from $\mathcal{K}$}{
    \If{$\ell \not\in \varphi$ and $\neg \ell \not\in \varphi$}{
      $\Pi^+ \leftarrow$ \textsc{LiftedLearn+Search}$(\mathcal{K},\; s{-}1,\; \varphi \vee \ell,\; \{ \rho^{(i)} : \ell|_{\rho^{(i)}} \neq 1 \})$ \\
      \If{$\Pi^+ \neq$ \texttt{None}}{
        $\Pi^- \leftarrow$ \textsc{LiftedLearn+Search}$(\mathcal{K},\; s,\; \varphi \vee \neg \ell,\; \{ \rho^{(i)} : \ell|_{\rho^{(i)}} \neq 0 \})$ \\
        \If{$\Pi^- \neq$ \texttt{None}}{
          \Return resolution proof combining $\Pi^+$ and $\Pi^-$ on pivot $\ell$
        }
      }
    }
  }
}

\Return \texttt{None}
\end{algorithm}

We now demonstrate that using subsumptive weakening and pivot literals is sufficient to derive a sound and complete algorithm for query validation. We formalize this procedure in the \textsc{LiftedLearn+Search} algorithm.

\begin{theorem}
Let $\varphi$ be a $\forall$-clause query, and $\mathcal{K}$ be a knowledge base consisting of a finite set of $\forall$-clauses.  $N$ is the number of different atoms appeared in the knowledge base. Let $\mathcal{P}$ be a distribution over models, and let $\Theta$ be a $(1 - \eta)$-concealing masking process with respect to CNFs up to size $2(eN/s-1)^{s-1}$ for the distribution $\mathcal{P}$.
Then on a sample of size $m \in \Theta\left(\frac{N^{s-1} \log N + \log \frac{1}{\delta}}{\eta\varepsilon}\right)$, and with running time $O\left(\frac{N^{2(s-1)}|\varphi|}{\eta\varepsilon}(N^{s-1}\log N + \log \frac{1}{\delta})\right)$ the algorithm with fixed parameter $s$ satisfies:

\begin{itemize}
 \item (\textbf{soundness})  If $[\mathcal{K} \Rightarrow \varphi]$ is not $(1-\varepsilon)$-valid,  with probability $1-\delta$, Algorithm~\ref{alg:lifted_learn_search} returns None.
 
    \item (\textbf{completeness}) If 
    there exists some perfectly valid CNF $\mathcal{H}$ for which there is a space-$s$ treelike resolution proof of $\varphi$ from $\mathcal{K} \wedge \mathcal H$, then Algorithm~\ref{alg:lifted_learn_search}  returns a proof of $\varphi$ from $\mathcal{K} \wedge \mathcal{H}_0$ for some CNF $\mathcal{H}_0$ of size $O(N^{s-1})$ that is $(1-\varepsilon)$-valid with respect to $\mathcal D$ with probability $1-\delta$. 
   
\end{itemize}

\end{theorem}

\begin{proof}

\medskip\noindent\textbf{(1) Soundness.}  By definition, Algorithm 1 produces a space-\(s\) normal treelike resolution proof. Let \(\mathcal{H}_{\text{bad}}\) be the CNF formed by the conjunction of all such leaf clauses in the proof. By Lemma~\ref{lem:bound} (Part 1), this CNF contains at most \(2(eN/s-1)^{s-1}\) clauses, hence bounded concealment applies.
Therefore, Theorem~\ref{thm:discount} guarantees that $\Theta$ will falsify this CNF with probability at least \(\eta \varepsilon\) if the CNF is not $1-\varepsilon$ valid.

We have $m$ draws of partial models in total.
Using the inequality \(1 - x \le e^{-x}\) for \(x \in (0,1)\), we obtain:
\[
(1 - \eta \varepsilon)^m \le e^{-m \eta \varepsilon}.
\]

Now choose
\[
m = \frac{N^{s-1} \log N + \log(1/\delta)}{\eta \varepsilon}.
\]
 Then:
\begin{align*}
e^{-m \eta \varepsilon}
&= e^{- (N^{s-1} \log N + \log(1/\delta))} \\
&= e^{- N^{s-1} \log N} \cdot e^{- \log(1/\delta)} \\
&= N^{- N^{s-1}} \cdot \delta.
\end{align*}

Hence $
\Pr[\mathcal{H}_{\text{bad}} \text{ is consistent with all samples}] \le \delta \cdot N^{-\Omega(N^{s-1})}.
$

Since weakening is not used in the clauses that are being asserted during learning, we do not need to bound the total number of distinct proofs. Applying a union bound over all such CNFs, the probability that \emph{any} non-\((1 - \varepsilon)\)-valid  CNF is consistent with the entire sample is at most \(\delta\). Hence, with probability at least \(1 - \delta\), the proof it returns has all the leaves at least $1-\varepsilon$ valid. But if the query is itself less than $1-\varepsilon$ valid, then such proof cannot exist and 
the algorithm returns \texttt{None}, establishing the soundness claim.

\medskip\noindent\textbf{(2) Completeness.} It remains to show that Algorithm~\ref{alg:lifted_learn_search} will indeed find a proof when a 1-valid space-$s$ treelike proof exists. Consider the (potentially exponential-size) formula $\widetilde{H}$ consisting of all clauses that are consistent with the obscured scenes $\rho^{(1)}, \ldots, \rho^{(m)}$. Note that weakening is not necessary for clauses from $\widetilde{H}$, as any clause derivable by weakening is already explicitly present in $\widetilde{H}$.

Algorithm~\ref{alg:lifted_learn_search} performs a systematic search for a normal space-$s$ treelike resolution proof, branching on literals that appear in $\mathcal{K}$. We claim that for any original 1-valid proof, there exists an equivalent one (possibly with new asserted clauses) that applies cuts only to pivot literals drawn from ground instances of $\mathcal{K}$.

If such a literal enables both branches of a resolution step to be verified using the data, then the resolved clause can itself be verified directly, without the need to branch on that literal. On the other hand, if at least one of the branches cannot be verified (i.e., requires weakening), then introducing a literal not in $\mathcal{K}$ is not helpful, because the resulting clause would not be derivable from $\mathcal{K} \land \widetilde{H}$ and would not be verified by the data.

Moreover, we do not need to introduce new variables when branching on clauses from $\mathcal{K}$. Although such clauses may contain universally quantified variables, renaming them to match variables already in the query’s scope merely produces a special case of the original clause. Since universally quantified statements are closed under substitution, this specialization remains entailed by the knowledge base and does not require augmenting the query with fresh variables.

Therefore, even if the original 1-valid proof uses literals outside the hypothesis, the algorithm can always find an equivalent proof that uses only literals from $\mathcal{K}$. This justifies restricting branching literals to those in the knowledge base, without loss of completeness in the class of verifiable proofs.

Furthermore, the standard analysis of space-bounded treelike resolution (e.g.,~\citet{Kullmann1999}) guarantees that the algorithm runs in time $O(m|\varphi|N^{2(s-1)})$ on a sample of size $m$. When a 1-valid space-$s$ treelike proof exists, it is consistent with all samples with probability 1.

Thus, whenever Algorithm~\ref{alg:lifted_learn_search} terminates early, it returns a proof that is consistent with the examples, and by soundness, the proof is $(1 - \varepsilon)$-valid. Moreover, if the algorithm explores the literals used in the known 1-valid proof (though it may succeed even earlier), it is guaranteed to find a valid derivation. Hence, the algorithm is complete for the class of 1-valid space-$s$ treelike proofs.
\end{proof}

\begin{example}
Consider the query:
\[
\varphi = \forall x \; \neg \texttt{Squirrel}(x) \lor \texttt{Grey}(x)
\]
and the knowledge base:
\begin{align*}
\mathcal{K} = \{ 
& \forall x \; \neg \texttt{Squirrel}(x) \lor \texttt{TreeDweller}(x), \\
& \forall x \; \neg \texttt{Squirrel}(x) \lor \texttt{NutEater}(x) \}
\end{align*}
which states that all squirrels live in trees and eat nuts.

Suppose we observe partial scenes where all tree-dwellers are grey. This validates:
\[
\forall x \; \neg \texttt{TreeDweller}(x) \lor \texttt{Grey}(x)
\]
from data, but falsifies $\neg \texttt{NutEater}(x) \lor \texttt{Grey}(x)$.

When the algorithm selects \texttt{TreeDweller}$(x)$ as a pivot literal, resolution succeeds using one clause from $\mathcal{K}$ and one statistically validated clause, constructing a lifted proof of $\varphi$, as illustrated in Figure \ref{fig:resolution-tree}.
\end{example}

\begin{figure}[t]
\centering
\begin{tikzpicture}[
  level distance=1.4cm,
  sibling distance=4.2cm,  
  every node/.style = {font=\scriptsize},
  edge from parent/.style={draw, -latex},
  treenode/.style = {align=center, inner sep=1pt}
  ]

\node [treenode] (root) {$\neg\texttt{Squirrel}(x) \vee \texttt{Grey}(x)$ \\ ($\varphi$)}
  child { node [treenode] (left) {$\neg\texttt{Squirrel}(x) \vee \texttt{TreeDweller}(x)$ \\ ($\mathcal{K}$)} 
    edge from parent node[left,xshift=-1mm,yshift=2mm] {}
  }
child { node [treenode, align=center] (right)
  {$\neg\texttt{TreeDweller}(x) \vee \texttt{Grey}(x)$\\($\data$)}
};

\end{tikzpicture}
\caption{Lifted resolution proof of $\neg\texttt{Squirrel}(x) \vee \texttt{Grey}(x)$ using both knowledge base and data.}
\label{fig:resolution-tree}
\vspace{-1.7em}
\end{figure}
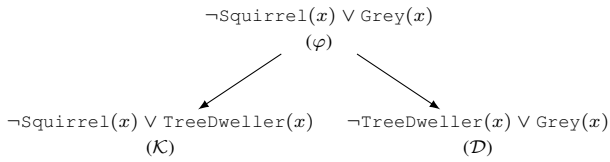

\subsection{Choosing the Space Bound}

Our algorithm requires the proof's space bound $s$ as input, which controls the tradeoff between the expressiveness of the proof system and the computational resources required for proof search. Larger values of \( s \) permit deeper and more complex derivations, potentially increasing completeness, but also lead to exponential growth in the number of candidate proofs (see Lemma~\ref{lem:bound}). We note that  space-bounded treelike resolution has another interpretation as an extension of the well-known \emph{unit-propagation} inference rule (where $s=2$) with a rule known as \emph{failed-literal} \citep{ansotegui2008measuring}: we temporarily assert a literal $\ell$ as a hypothesis, and if the formula can be refuted, we infer $\neg\ell$. Then \( s\) corresponds to a bound on the number of literals that may be simultaneously asserted while invoking the failed-literal rule; algorithms for automating space-bounded treelike resolution essentially follow the strategy implicitly described by this characterization, searching for literals on which to invoke the failed-literal rule.

In practice, \( s \) should be chosen based on available memory and computational budget, as well as the anticipated complexity of the domain. \citet{ansotegui2008measuring} also examined the value of \( s\) in many industrial instances of propositional satisfiability, and found that it is significantly smaller than for random instances in the hard regime at the phase transition, suggesting that it may capture what makes real-world instances easy.

Moreover, our algorithm is anytime: if no proof is found for a given \( s \), the parameter can be incrementally increased, enabling flexible exploration of the proof space.

\section{Related Work}
Our work targets an integration of symbolic reasoning with statistical learning—two paradigms that have largely developed in parallel. The closest line to ours begins with \citet{de1993theory}, who introduced a theory of clausal discovery divorced from any underlying data distribution. Soon afterwards, \citet{khardon1997learning} formalized the “learning to reason” framework, establishing a principled synthesis of learning and proof, but under a strong restriction: the learner reasons in a fully observed setting. Subsequently, \citet{valiant1999robust} proposed PAC-Semantics, relaxing classical validity to accommodate imperfect rules as long as they hold with high probability under the data distribution. Building on these foundations, \citet{Juba2013Implicit} introduced \textit{implicit learning to reason}, which concentrates computational effort on what is needed for a particular query, yielding large savings in time and space. Later work extended this perspective toward limited forms of first-order reasoning \citep{belle2019implicitly}, as discussed in the Introduction. More recently, \citet{juba2019query} developed \textit{query-driven learning to reason}, strengthening the implicit approach by explicitly identifying the supporting proof structure for a given query. This explicit proof output is especially important in settings that require interpretability, human oversight, or supervised refinement. Nevertheless, this work was  primarily developed in the propositional setting, and does not address our challenge in lifted proofs for universally quantified queries under partial observability.

Outside the ``learning to reason" line, several other research programs are conceptually related but do not focus on efficiency and soundness or completeness guarantees as we do. We give an overall review as follows: 

\paragraph{Statistical Relational Learning (SRL).}

SRL frameworks build probabilistic models over relational domains in order to support uncertainty-aware inference. Representative formalisms include Bayesian Logic Programs~\citep{kersting20071}, Markov Logic Networks~\citep{richardson2006markov}, and probabilistic logic programming~\citep{de2007problog},  which traditionally compiled to or were parameterized by graphical models, and more recent approaches have used a Weighted First-Order Model Counting (WFOMC) formulation for inference \citep{braz2005lifted,van2011lifted,gogate2016probabilistic}. While \emph{lifted probabilistic inference}~\citep{poole2003first,kersting2012lifted,van2021introduction} exploits symmetries to improve scalability at inference time, two core challenges remain during learning in particular. First, structure learning is  computationally difficult: for expressive SRL formalisms, it is typically combinatorial and intractable \citep{chickering1996learning}, with practical methods  relying on restrictive hypothesis classes and heuristic search~\citep{kok2005learning}.  Second, handling missing data is commonly approached via maximum-likelihood or MAP objectives with latent variables, for which expectation--maximization (EM)~\citep{meert2008learning} and related iterative procedures are standard. But these methods generally offer only weak guarantees: they need not find a global optimum, and even convergence to a stationary point can be slow or sensitive to initialization \citep[Section 19.4]{koller2009probabilistic}. As a result, the end-to-end pipeline is hard to scale efficiently, and its outputs are not accompanied by proof-style certificates or optimality guarantees.


\paragraph{Inductive Logic Programming (ILP).}
ILP systems~\citep{muggleton1994inductive} aim to induce logical rules from examples consistent with background knowledge, without explicitly modeling uncertainty. While classical systems like Aleph~\citep{srinivasan2001aleph} and Progol~\citep{muggleton1995inverse} are effective at rule generalization, they are not designed for validation or query answering with a distribution in mind.  The most closely related ILP work is that of \citet{DERAEDT1994375},
which learns small first-order clauses with PAC learning guarantees relative to a
finite model, rather than the infinite model we consider, where such guarantees are
possible only under bounded-concealment assumptions. Additionally, extensions to probabilistic ILP~\citep{de2008probabilistic} inherit the limitations of SRL, including the challenges with structure learning and handling incomplete observations.

\paragraph{Neuro-symbolic approaches.}
Neuro-symbolic (N-S.) methods combine neural architectures with symbolic logic, often by embedding predicates into continuous representations or approximating inference steps. Notable examples include DeepStochLog~\citep{winters2022deepstochlog}, Lifted Relational Neural Networks~\citep{sourek2018lifted}, and Neural Logic Machines~\citep{dong2019neural}. Unlike our approach, these systems enable learning from weak supervision but do not produce interpretable proofs or offer formal correctness guarantees. They also do not address the verification of universally quantified queries or model partial observability with principled assumptions. 

\section{Conclusion and Future Work}

In this work, we introduce the first query-driven learning to reason framework that provides simultaneous efficiency and correctness guarantees at the lifted level. Our approach bridges symbolic reasoning and statistical learning by enabling universally quantified clauses to be validated from partial observations.

Several directions remain open.  
First, we restrict attention to Boolean domains,  which aligns naturally with resolution but leaves open extensions to richer lifted settings, such as lifted real linear arithmetic \citep{kersting2017relational} (or quadratic \citep{mladenov2017lifted}) or probabilistic relational logics \citep{Ge2025Poly}, where the expressions involve lifted families of real-valued variables. 
In a closely related vein, our work only considers a simple kind of masking where attributes are simply present or missing. If we consider the extension to real-valued expressions, then we may correspondingly consider other forms of uncertainty in our partial examples, such as intervals, along the lines of \citet{rader2021learning} who considered implicit learning in the propositional setting. It would be interesting to extend query-driven explicit learning to handle such partial information models.

Second, we only guarantee completeness when there is a \emph{perfectly} valid knowledge base that entails the query -- even one counterexample rules out a prospective clause from inclusion in our learned knowledge -- which raises the question of how we can handle noisy or imperfect data. While one can relax our algorithm by allowing a small fraction of falsifying examples before rejecting clauses, in general this only yields a bound on the soundness of the query that scales with the number of learned clauses (using Proposition \ref{classical-inf-bound}); can we obtain a tighter guarantee?

Finally, it would be fruitful to develop  practical applications.  
Promising application domains include knowledge extraction from structured or semi-structured sources, as well as planning tasks where lifted reasoning about actions and goals can benefit from integrating learned rules and partial observations.



\bibliographystyle{kr}
\bibliography{kr-sample}

\end{document}